\documentclass[mmnp]{edpsmath}
\usepackage{graphicx}%
\usepackage{multirow}%
\usepackage{amsmath,amssymb,amsfonts}%
\usepackage{amsthm}%
\usepackage{mathrsfs}%
\usepackage[title]{appendix}%
\usepackage{xcolor}%
\usepackage{textcomp}%
\usepackage{manyfoot}%
\usepackage{booktabs}%
\usepackage{algorithm}%
\usepackage{algorithmicx}%
\usepackage{algpseudocode}%
\usepackage{listings}%
\usepackage{pdflscape}
\usepackage[colorlinks=true]{hyperref}
\hypersetup{urlcolor=blue, citecolor=red}
\newcommand{\poubelle}[1]{}

\theoremstyle{thmstyleone}%
\newtheorem{theorem}{Theorem}
\newtheorem{proposition}[theorem]{Proposition}%
\newtheorem{lemma}[theorem]{Lemma}

\theoremstyle{thmstyletwo}%
\newtheorem{remark}{Remark}%

\theoremstyle{thmstylethree}%
\newtheorem{definition}{Definition}%

\begin{document}
\title[Dynamical analysis of a microbial decomposition model in soil]{Dynamical analysis and numerical simulation of a reaction-diffusion model for microbial decomposition of organic matter in 3D soil structure}
\author{Mohammed Elghandouri}\address{LMDP Laboratory, Cadi Ayyad University, B.P: 2390, Marrakesh, Morocco / IRD, UMMISCO, F-93143, Bondy, Paris, France. E-mail: medelghandouri@gmail.com}
\author{Ezzinbi Khalil}\address{LMDP Laboratory, Cadi Ayyad University, B.P: 2390, Marrakesh, Morocco / IRD, UMMISCO, F-93143, Bondy, Paris, France. E-mail: ezzinbi@uca.ac.ma}
\author{Klai Mouad}\address{LMDP Laboratory, Cadi Ayyad University, B.P: 2390, Marrakesh, Morocco / IRD, UMMISCO, F-93143, Bondy, Paris, France. E-mail: klaimouadklai@gmail.com}
\author{Olivier Monga}\address{IRD, UMMISCO, F-93143, Bondy, Paris, France. E-mail: olivier.monga@ird.fr}
\date{July 20, 2024}
\begin{abstract} Microbial decomposition of organic matter in soil is a fundamental process in the global carbon cycle, directly influencing soil health, fertility, and greenhouse gas emissions. This paper presents a dynamic analysis and numerical simulation of a reaction-diffusion model that describes microbial decomposition of organic matter within a three-dimensional soil structure. We explore the interactions between microbial biomass and organic substrates, and diffusion of different compounds through the soil matrix using nonlinear parabolic partial differential equations. Our study provides proofs for the existence and uniqueness of solutions, as well as an analysis of asymptotic behavior. Notably, our investigation reveals the presence of a global attractor, where any solution, regardless of initial conditions, tends to converge. To illustrate the practical implications of our findings, we have developed a numerical tool to simulate the long-term behavior of the system with reasonable computational expense. This tool provides a visual proof of the global attractor for a validated set of biological parameters in a real sandy loam soil sample captured using 3D tomography imagery. \end{abstract}
\maketitle
\section*{Introduction}
Microbial activity in soil is vital for preserving soil health and ensuring ecosystem functionality. Microorganisms play key roles in nutrient cycling, breaking down organic matter, and fostering plant growth \cite{intr1,intr2}. The vast diversity and intricate interactions within soil present a challenge to fully grasp its complexity. Furthermore, soil's heterogeneous and multifaceted nature makes directly observing and measuring microbial activity difficult \cite{intr3}. Traditional methods, such as culturing and microscopy, are not only time-consuming but also offer limited insights. Hence, there is a pressing need for alternative methods for studying soil microbial activity \cite{intr4}. Over the past two decades, numerous studies have concentrated on modeling the dynamics within the complex soil structure. These efforts seek to mimic how soil components move and interact in porous media and fractured environments. Several significant contributions have emerged from these studies, with models typically integrating transport mechanisms with sophisticated reaction (transformation) processes to simulate the activity of microorganisms and compounds within the soil matrix. Significant attention has been given to computational models, focusing on the development of numerical tools to simulate transformation and diffusion processes in the complex pore space of soil. Various methodologies have been employed to describe chemical transport phenomena, primarily the Lattice-Boltzmann Method (LBM) \cite{intr5} and pore network geometrical models (also known as morphological models) \cite{intr6}. On the other hand, the mathematical analysis of reaction-diffusion models has only been investigated in a limited number of studies. One notable example is found in \cite{7}, where the authors used nonlinear parabolic Partial Differential Equations (PDEs) to model the microbial decomposition of organic matter in soil. This approach allowed for a detailed representation of the biochemical processes occurring within the soil matrix. To solve these complex equations, the authors employed FreeFEM++, a powerful finite element software, to simulate the model based on the variational formulation of the PDEs. However, further research is needed to analyse and refine these mathematical models.

In the present paper, we investigate a nonlinear parabolic PDE model describing microbial activity in soil and provide a formal analysis of the model, proving the existence and uniqueness of a global solution. Then, we prove the existence of a global attractor, where any solution starting from any initial distribution tends to converge. Furthermore, we provide a numerical tool to simulate the long-term behavior of the system in a reasonable time using an optimal sphere network that approximates the pore space of soil. Our research begins by describing a biological model that represents the dynamic processes occurring within the soil. This model encapsulates key biological parameters and system behaviors, shedding light on organic matter decomposition by microorganisms and the production of mineralized CO$_2$, and taking into account the diffusion of different compounds through the soil matrix.
To rigorously analyze this biological system, we employ Partial Differential Equations (PDEs) as our mathematical tool of choice. Specifically, we formulate a set of parabolic nonlinear PDEs that elegantly capture the diffusion and transformation processes expressing the degradation of organic matter by microorganisms taking place within the 3D soil structure. These equations are accompanied by Neumann boundary conditions, reflecting the system's behavior at its boundaries. Before diving into the core of our analysis, we provide an overview of the mathematical tools and concepts essential to our proof of the global attractor for the PDE model \cite{5,pliss,yoshizawa,milor,sell}. These preliminary insights based on dynamical systems theory set the stage for the subsequent sections, where we unveil the intriguing properties of the system. One of our primary objectives is to establish the global existence of solutions to the PDE model. We demonstrate that the problem possesses a unique positive mild solution defined over an unbounded time interval, $ [ 0,+\infty [$. This foundational result forms the basis for our exploration of the system's long-term behavior.  At the core of our research, we reveal the presence of a global attractor which is a compact, connected, and invariant subset of the positive points comprising the system's phase space. Importantly, this global attractor possesses the remarkable capacity to attract and encapsulate all bounded subsets of the phase space, thereby affording invaluable insights into the long-term behavior of the system. To enhance the robustness of our theoretical findings, we employ a geometric approach to model the intricate pore space within the soil structure. This methodology involves adapting the PDE model to fit the novel framework of the pore network model, a technique that has been previously validated \cite{17}. The approach exhibits exceptional capabilities in simulating the long-term behavior of the system and considered as a powerful alternative of Lattice Boltzmann method in simulating diffusion in complex geometries. 
In our investigations, we harness this method to explore a range of real parameters and scenarios on a real soil sample captured using computed tomographic imagery. 

The description of this work is organized as follows. In Section \ref{Sec 2}, we present our biological model. In Section \ref{Sec 3}, we introduce our model using partial differential equations (PDEs). Section \ref{Sec 4} revisits essential properties related to global attractors for partial differential equations and discusses properties crucial for proving our main results. In Section \ref{Sec 5}, we delve into the global existence of our PDEs system. Section \ref{Sec 6} establishes the existence of a global attractor for our model. Section \ref{Sec 7} is dedicated to presenting numerical simulations for visualizing the attractor. Finally, we showcase and discuss our results in Section \ref{Sec 8}.
\section{General description of the biological model} \label{Sec 2}
\noindent

Our objective is to simulate and model the process of organic carbon mineralization by microorganisms within the soil pore space. We improve the methodology outlined in \cite{11,12,7} to simulate the organic matter decomposition by microorganisms in the soil pore space by solving a PDEs system. Microorganisms degrade organic matter in accordance with the laws of supply and demand. The method presented in \cite{7} is based on the law of conservation of mass to obtain a reaction-diffusion model that describes the diffusion of organic matter. In our approach, we assume that the decomposition of soil organic matter involves five key biological components. 

Microbial activity is governed by the diffusion of various compounds within the soil's pores and by transformation processes. We model the dynamics of microbial decomposition using five key compounds, as illustrated in figure \ref{Fig 0}: 
\begin{itemize}
	\item Microbial biomass (MB): Represents the mass of microorganisms in the sample.
	\item 	Microbial respiration (CO2): Indicates the produced carbon dioxide through microbial decomposition, indicative of microbial growth.
	\item Fresh organic matter (FOM): Derived from recently added or deposited plant and easily decomposable 
	\item Soil organic matter (SOM): Consists of various organic compounds in different stages of decomposition, originating from biomass turnover and less accessible to decomposition.
	\item Dissolved organic matter (DOM): Refers to organic compounds dissolved in soil water, originating from hydrolyze of FOM and SOM and biomass recycling, available for microbial uptake or transport within the pore space.
\end{itemize}
Dissolved organic matter arises from the decomposition of both slow-decomposing soil organic matter and fast-decomposing fresh organic matter. Microorganisms grow through the assimilation of DOM and respire by producing CO2. Upon death, they are transformed back into organic matter. The biological process is summarized in Figure \ref{Fig 0}.
\begin{figure}[h]
	\centering
	\includegraphics[scale=0.4]{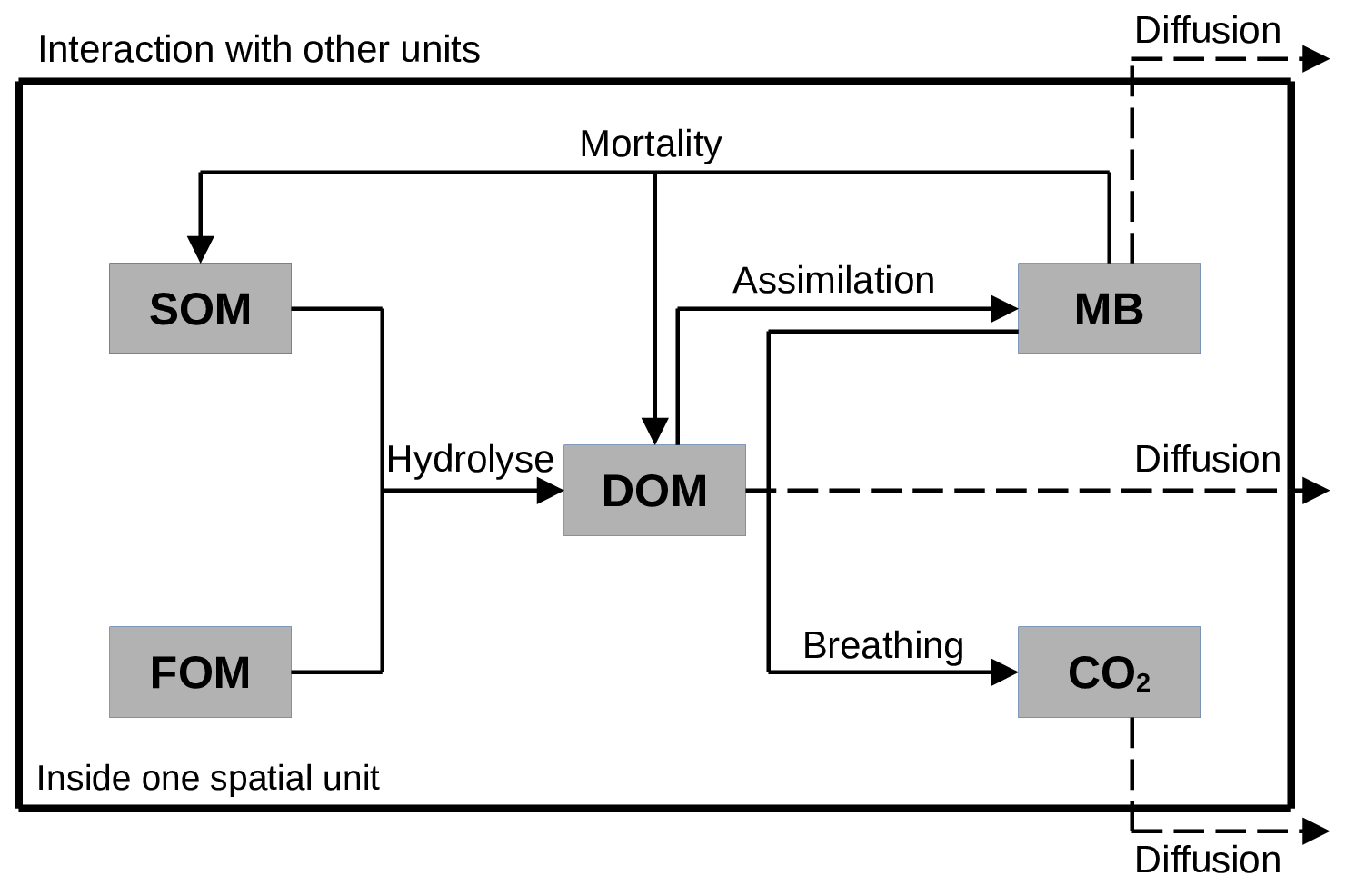}
	\caption{Biological processes simulated inside one spatial unit and the
		interactions with external units.}
	\label{Fig 0}
\end{figure}

\section{Mathematical formulation of the model} \label{Sec 3}
\noindent

Let $\Omega\subset \mathbb{R}^{3}$ be the domain representing the 3D soil pore space, $t\geq 0$ the time given, and $x=(x_1,x_2,x_3)^{T}\in\Omega$ be a point of the pore space. We denote by $b(x,t)$, $n(x,t)$, $m_1(x,t)$, $m_2(x,t)$, and $c(x,t)$ the densities of the microorganisms biomass (MB), the dissolved organic matter (DOM), the soil organic matter (SOM), the fresh organic matter (FOM), and the carbon dioxide (CO$_2$) at position $x\in \Omega$ and time $t\geq 0$ respectively.

Let $\mathcal{V}$ be an individual volume in $\Omega$. The density evolution of the microbial biomass (MB) in $\mathcal{V}$ depends on the diffusion of the microorganisms, the growth of the microorganisms, the mortality of the microorganisms, and the breathing of the microorganisms. It is assumed that the microorganisms biomass (MB) consumes the dissolved organic matter to growth. The variation of the microbial biomass density is described by the following equation:
\begin{equation*}
\dfrac{\partial b}{\partial t}=D_b\Delta b+\dfrac{Kn}{K_b+n}b-(\mu+\eta)b,
\end{equation*}
where $D_b$ is the microbial decomposer's diffusion coefficient, $K$ is the maximal growth rate, $K_b$ is the half-saturation constant, $\mu$ is the mortality rate, and $\eta$ is the breathing rate.

The density evolution of the dissolved organic matter (DOM) in $\mathcal{V}$ depends on the (DOM) diffusion, the assimilation of the dissolved organic matter by the microbial biomass decomposer's, the microbial biomass decomposer's mortality, and the transformation of the soil organic matter and the fresh organic matter. The equation that describe the (DOM) density variation is the following:
\begin{equation*}
\dfrac{\partial n}{\partial t}=D_n\Delta n-\dfrac{Kn}{K_b+n}b+c_1m_1+c_2m_2+\rho\mu b,
\end{equation*}
where $D_n$ is the (DOM) diffusion coefficient, $c_1$ is the maximal transformation rate of the soil organic matter, $c_2$ is the maximal transformation rate of the fresh organic matter, and $\rho \mu$ is the transformation rate of deceased microbial biomass decomposer's into dissolved organic matter (DOM).

The variation in soil organic matter (SOM) quantity arises from the conversion of a portion of (SOM) into dissolved organic matter and from the mortality of microbial biomass. Thus, the density evolution of the soil organic matter (SOM) in $\mathcal{V}$ is presented by the following equation:
\begin{equation*}
\dfrac{\partial m_1}{\partial t}=-c_1m_1+(1-\rho)\mu b.
\end{equation*}

The density evolution of the fresh organic matter (FOM) in $\mathcal{V}$ is presented by the following equation:
\begin{equation*}
\dfrac{\partial m_2}{\partial t}=-c_2m_2.
\end{equation*}  

The evolution of the dioxide carbon (CO$_2$) density in $\mathcal{V}$ depends on its diffusion and its production by the microbial biomass decomposer's during breathing. The (CO$_2$) quantity variation is represented by the following equation:
\begin{equation*}
\dfrac{\partial c}{\partial t}= D_c\Delta c+\eta b,
\end{equation*}
where $D_c$ is the carbon dioxide diffusion coefficient.

Therefore, the microbial decomposition of organic matter in soil is described by the following PDEs system:
\begin{equation}
\left\{\begin{array}{l}
\dfrac{\partial b}{\partial t}=D_b\Delta b+\dfrac{Kn}{K_b+n}b-(\mu+\eta)b,\\ \\
\dfrac{\partial n}{\partial t}=D_n\Delta n-\dfrac{Kn}{K_b+n}b+c_1m_1+c_2m_2+\rho\mu b,\\ \\
\dfrac{\partial m_1}{\partial t}=-c_1m_1+(1-\rho)\mu b,\\ \\
\dfrac{\partial m_2}{\partial t}=-c_2m_2,\\ \\
\dfrac{\partial c}{\partial t}= D_c\Delta c+\eta b.
\end{array}\right.
\label{eq 1.0}
\end{equation}

On the border $\partial \Omega$ of $\Omega$, we use Neuman boundary conditions, which means that the flow is null on $\partial \Omega$ for all variables, i.e,
\begin{equation}
\dfrac{\partial b}{\partial \overrightarrow{n}} = \dfrac{\partial n}{\partial \overrightarrow{n}}=\dfrac{\partial m_1}{\partial \overrightarrow{n}}=\dfrac{\partial m_2}{\partial \overrightarrow{n}}=\dfrac{\partial c}{\partial \overrightarrow{n}}= 0, \quad on \quad \partial\Omega\times (0,T).
\label{eq 2.0}
\end{equation}

For the initial conditions  in $\Omega$, we denote by
\begin{equation}
(b(0,\cdot),n(0,\cdot),m_1(0,\cdot),m_2(0,\cdot),c(0,\cdot))=(b_0,n_0,m_{1,0},m_{2,0},c_0).
\label{eq 3.0}
\end{equation}

To simplify \eqref{eq 1.0}-\eqref{eq 3.0} into a vector form equation, we define:  the state system $u$ as:
\begin{center}
	$u= (u_1,u_2,u_3,u_4,u_5)^{T} = (b,n,m_1,m_2,c)^{T}$.
\end{center}
The initial state $u_0$ as:
\begin{center}
	$u_0=(b_0,n_0,m_{1,0},m_{2,0},c_0)^{T}$.
\end{center}
The diffusion coefficients matrix $\underline{D}$ as:
\begin{center}
	$\underline{D}=\begin{pmatrix}
	D_b & 0 & 0 & 0 & 0 \\
	0 & D_n & 0 & 0 & 0 \\
	0 & 0 & 0 & 0 & 0 \\
	0 & 0 & 0 & 0 & 0 \\
	0 & 0 & 0 & 0 & D_c \\
	\end{pmatrix}$.
\end{center}
The reaction diffusion terms vector $F(u)$ is given by:
\begin{center}
	$F(u)=(F_1(u),F_2(u),F_3(u),F_4(u),F_5(u))$,
\end{center}
where 
\begin{equation*}
\left.\begin{array}{l}
F_1(u)=\dfrac{Ku_2}{K_b+u_2}u_1-(\mu+\eta)u_1,\\ \\
F_2(u) =-\dfrac{Ku_2}{K_b+u_2}u_1+c_1u_3+c_2u_4+\rho \mu u_1,\\ \\
F_3(u)=-c_1u_3+(1-\rho)\mu u_1,\\ \\
F_4(u)=-c_2u_4,\\ \\
F_5(u)= \eta u_1.
\end{array}\right.
\end{equation*}

Finally, equation \eqref{eq 1.0}-\eqref{eq 3.0} can be reduced to:
\begin{equation}
\left\{\begin{array}{l}
\vspace{0.3cm}
\dfrac{\partial u}{\partial t}= \underline{D} \Delta u+F(u), \quad in\quad \Omega\times(0,T)\\ 
\vspace{0.3cm}
\dfrac{\partial u}{\partial \overrightarrow{n}} = 0, \quad on \quad \partial\Omega\times (0,T)\\ 
u(0)=u_0, \quad in \quad \Omega.
\end{array}\right.
\label{eq 1.1}
\end{equation}  
\subsection{Why do we need to look for an attractor?}
\noindent

Solving the equation $F(u) = 0$ allows us to determine the set of constant equilibrium points of equation \eqref{eq 1.1}, which is given by:
\begin{center}
	$S=\left\{ (0,u_2^{*},0,0,u_5^{*}): \hspace{0.1cm} u_2^{*},u_5^{*}\in \mathbb{R} \right\}$.
\end{center}
This notation indicates that the equilibrium points are characterized by specific values for the variables $u_2^*$ and $u_5^*$, both of which belong to the set of real numbers. An important observation is that these equilibrium points are not isolated. This lack of isolation has implications for our analytical methods. Specifically, it means that we cannot employ techniques like linearization or Lyapunov's method, which are often used to analyze the behavior of systems near isolated equilibrium points. Furthermore, we do not expect to observe convergence to the equilibrium points, which makes understanding the system's behavior as it approaches equilibrium challenging. Given these challenges, our emphasis lies in establishing the existence of a global attractor for equation \eqref{eq 1.1}.

\section{Basic working tools} \label{Sec 4}
\noindent

In the sequel, we recall some useful properties on global attractor for partial differential equations. For more details, we refer  to \cite{5}.

\begin{definition} \cite{5}
	A semiflow $(\mathcal{U}(t))_{t\geq 0}$ on a complete metric space $\mathcal{X}$ is a one parameter family maps $\mathcal{U}(t):\mathcal{X}\to \mathcal{X}$, parameter $t\in \mathbf{R}^{+}$, such that
	\begin{enumerate}
		\item[i)]  $\mathcal{U}(0)=id_{\mathcal{X}}$ ($id_{\mathcal{X}}$ the identity map on $\mathcal{X}$).
		\item[ii)] $\mathcal{U}(t+s)=\mathcal{U}(t)\mathcal{U}(s)$ for $t,s\geq 0$.
		\item[iii)] $t\to \mathcal{U}(t)x$ is continuous for all $x\in \mathcal{X}$.
	\end{enumerate}
\end{definition}

\begin{definition} 
	Let $(\mathcal{U}(t))_{t\geq 0}$ be a semiflow on a complete metric space $\mathcal{X}$.
	\begin{enumerate}
		\item[i)] We say that $(\mathcal{U}(t))_{t\geq 0}$ is bounded, if it takes bounded sets into bounded sets.
		\item[ii)] We say that $B\subset\mathcal{X}$ is invariant under $(\mathcal{U}(t))_{t\geq 0}$ if $\mathcal{U}(t)B= B$ for $t\geq 0$.
	\end{enumerate}
\end{definition}

\begin{definition} Let $(\mathcal{U}(t))_{t\geq 0}$ be a semiflow on a complete metric space $\mathcal{X}$.
	We say that $(\mathcal{U}(t))_{t\geq 0}$ is point dissipative if there is a bounded set $B\subset \mathcal{X}$ which attracts each point of $\mathcal{X}$ under $(\mathcal{U}(t))_{t\geq 0}$.
\end{definition}

\begin{theorem} \cite{5} Let $(\mathcal{U}(t))_{t\geq 0}$ be a semiflow on a complete space $\mathcal{X}$ with the metric $\Vert\cdot\Vert$. Assume that there exists $L>0$ such that for all $x_0\in \mathcal{X}$, $\limsup\limits_{t\to +\infty}\Vert \mathcal{U}(t)x_0\Vert\leq L$, then $(\mathcal{U}(t))_{t\geq 0}$ is point dissipative.
\end{theorem}

\begin{definition} 	Let $(\mathcal{U}(t))_{t\geq 0}$ be a semiflow on a complete metric space $(\mathcal{X},d)$.  A subset $B$ of $\mathcal{X}$ is said to attract  $C\subset \mathcal{X}$ if $d(\mathcal{U}(t)C,B)\to 0$ as $t\to +\infty$, 
	where 
	\begin{center}
		\begin{center}
			$d(\mathcal{U}(t)C,B):=\inf\{ d(\mathcal{U}(t)x,y); \hspace{0.1cm} x\in B \text{ and } y\in C \}$.
		\end{center}
	\end{center}
\end{definition}

\begin{definition} \cite{5} Let $\mathcal{X}$ be a complete metric space and $(\mathcal{U}(t))_{t\geq 0}$ be a semiflow on $\mathcal{X}$. Let $\mathcal{A}$ be a subset of $\mathcal{X}$. Then, $\mathcal{A}$ is called a global attractor for $(\mathcal{U}(t))_{t\geq 0}$ in $\mathcal{X}$ if $\mathcal{A}$ is a closed and bounded invariant set of $\mathcal{X}$ that attracts every bounded set of $\mathcal{X}$.
\end{definition}

The following Theorem shows the existence of global attractor for point dissipative semiflows.

\begin{theorem} \cite[Theorem 4.1.2, page 63]{5} Let $(\mathcal{U}(t))_{t\geq 0}$ be a semiflow on a complete metric space $\mathcal{X}$. Suppose that $\mathcal{U}(t)$ is bounded for each $t\geq 0$. If $(\mathcal{U}(t))_{t\geq 0}$ is point dissipative, then it has a connected global attractor.
	\label{Theorem 4.1.2}
\end{theorem}

In the sequel, we need the following Lemma.

\begin{lemma}  \cite[Theorem 11.3 with $h=0$, page 99]{2} \itshape Let $v$ be a real, continuous and nonnegative function such that
	\begin{center}
		$v(t)\leq c+\displaystyle\int_{t_0}^{t}w(t,s)v(s)ds$, \text{ for } $t\geq t_0$,
	\end{center}
	where $c>0$, $w(t,s)$ is continuously differentiable in $t$ and continuous in $s$ with $w(t,s)\geq 0$ for $t\geq s\geq t_0$.
	Then,
	\begin{center}
		$v(t)\leq c\exp\left(\displaystyle\int_{t_0}^{t}\left[w(s,s)+\displaystyle\int_{t_0}^{s}\dfrac{\partial w(s,r)}{\partial s} dr\right]ds\right)$, \text{ for } $t\geq t_0$.
	\end{center}
	\label{lem 2.2}
\end{lemma}

\section{Global existence for equation \eqref{eq 1.1}}\label{Sec 5}
\noindent

The space $(\mathbb{L}^{2}(\Omega))^{5}$ endowed with the norm $\Vert \cdot\Vert_{ (\mathbb{L}^{2}(\Omega))^{5}}=\sum\Vert \cdot\Vert_{\mathbb{L}^{2}(\Omega)}$ issue from the inner product $(\cdot,\cdot)$ defined by 
\begin{center}
	$(v,w)=\sum\limits_{i=1}^{5}\langle v_i,w_i\rangle_{\mathbb{L}^{2}(\Omega)}=\sum\limits_{i=1}^{5}\displaystyle\int_{\Omega}v_i(x)w_i(x)dx$,
\end{center} 
for $v=(v_1, v_2,v_3,v_4,v_5)$, $w=(w_1, w_2,w_3,w_4,w_5)\in (\mathbb{L}^{2}(\Omega))^{5} $  is a Hilbert space.

In order to rewrite equation (\ref{eq 1.1}) in an abstract form, we define de operator $(A,D(A))$ by 
\begin{equation*}
\left\{\begin{array}{l}
D(A)=(H_0^2(\Omega))^{2}\times (\mathbb{L}^{2}(\Omega))^{2}\times H_0^2(\Omega)\\
Av=\underline{D} \Delta v, \text{ for } v\in D(A),
\end{array}\right.
\end{equation*}
where
\begin{center}
	$H^{2}_0(\Omega)=\left\{w\in \mathbb{L}^{2}(\Omega):\hspace{0.1cm}such \hspace{0.1cm} that\hspace{0.1cm}\nabla w, \hspace{0.1cm}\Delta w\in \mathbb{L}^{2}(\Omega), \hspace{0.1cm} and \hspace{0.1cm} \dfrac{\partial w}{\partial \overrightarrow{n}}=0\right\}$.
\end{center}
Then, equation (\ref{eq 1.1}) can be transformed to the following reaction-diffusion form:
\begin{equation}
\left\{\begin{array}{l}
u'(t)=Au(t)+F(u(t)), \text{ for } t\geq 0\\
u(0)=u_0.
\end{array}\right.
\label{eq 3.1}
\end{equation}
Let $\lambda>0$. Then, equation $(\ref{eq 3.1})$ can be written in the following equivalent form: 
\begin{equation}
\left\{\begin{array}{l}
u'(t)=(A-\lambda I)u(t)+G(u(t)), \text{ for } t\geq 0\\
u(0)=u_0,
\end{array}\right.
\label{eq 3.2}
\end{equation} 
where $G(v)=F(v)+\lambda v$ for $v\in $ $(\mathbb{L}^{2}(\Omega))^{5}$. 

In the next, we denote by 
\begin{center}
	$X:=(\mathbb{L}^{2}(\Omega))^{5}$ endowed with the norm $\Vert \cdot\Vert:=  \Vert \cdot\Vert_{(\mathbb{L}^{2}(\Omega))^{5}}$,
\end{center}
and
\begin{center}
	$\Lambda_{+}:=\left\{v=(v_1, v_2,v_3,v_4,v_5)\in X: \hspace{0.1cm} v_i\geq 0, \text{ for }  i=1,\cdots,5\right\}$.
\end{center}
\begin{remark} We have 
	\begin{eqnarray*}
		G(v)&=& F(v)+\lambda v=\begin{pmatrix}
			\dfrac{Kv_2}{K_b+v_2}v_1+ \left(\lambda-(\mu+\eta)\right)v_1\\
			c_1v_3+c_2v_4+\left(\lambda - \dfrac{Kv_2}{K_b+v_2}\right)v_1+\rho \mu v_1\\
			\left(\lambda -c_1\right)v_3+(1-\rho)\mu v_1\\
			\left(\lambda-c_2\right) v_4\\
			\eta v_1+\lambda v_5
		\end{pmatrix},
	\end{eqnarray*}
	for each $v\in \Lambda_{+}$. Since, $\dfrac{Kz}{K_b+z}\leq K$ for each $z\in \mathbb{R}^{+}$, if $\lambda>0$ is sufficiently large, then $G(v)\geq 0$ for each $v\in \Lambda_{+}$.
	\label{rem 2}
\end{remark}
\begin{lemma} \itshape There exists a positive constant $C_1>0$ such that $\Vert F(v)\Vert \leq C_1\Vert v\Vert$  for each  $v\in \Lambda_{+}$.
	\label{lem 3.1}
\end{lemma}
\begin{proof} Let $v=(v_1, v_2,v_3,v_4,v_5)\in \Lambda_{+}$. Since,
	\begin{equation*}
	\left.\begin{array}{l}
	\vert F_1(v)(x)\vert^{2}\leq 2\left( K^{2}+(\mu+\eta)^{2}\right)v_1(x)^{2};\\ \\ 
	\vert F_2(v)(x)\vert^{2}\leq 4\left[ (K^{2}+(\rho \mu)^2)v_1(x)^{2}+ c_1^2v_3(x)^2+c_2^2v_4(x)^2\right];\\ \\
	\vert F_3(v)(x)\vert^{2}\leq 2\left[  c_1^2v_3(x)^2+((1-\rho)\mu)^2v_1(x)^2\right];\\ \\
	\vert F_4(v)(x)\vert^{2}\leq c_2^2v_4(x)^2 ;\\ \\
	\vert F_5(v)(x)\vert^{2}\leq \eta^2v_1(x)^2.
	\end{array}\right.
	\end{equation*}
	It follows that
	\begin{eqnarray*}
		\Vert F(v)\Vert &=& \sum_{i=1}^{5}\Vert F_i(v)\Vert_{\mathbb{L}^{2}(\Omega)}\\ 
		&=& \sum_{i=1}^{5}\left(\displaystyle\int_{\Omega}\vert F_i(v)(x)\vert^2 dx\right)^{1/2}\\ 
		&=& \left(\displaystyle\int_{\Omega}2\left( K^{2}+(\mu+\eta)^{2}\right)v_1(x)^{2} dx\right)^{1/2}\\ 
		& & + \hspace{0.1cm} \left(\displaystyle\int_{\Omega}4\left[ (K^{2}+(\rho \mu)^2)v_1(x)^{2}+ c_1^2v_3(x)^2+c_2^2v_4(x)^2\right] dx\right)^{1/2}\\
		& & + \hspace{0.1cm} \left(\displaystyle\int_{\Omega}2\left[  c_1^2v_3(x)^2+((1-\rho)\mu)^2v_1(x)^2\right] dx\right)^{1/2}+ \left(c_2^2\displaystyle\int_{\Omega}v_4(x)^2dx\right)^{1/2}\\
		& & + \hspace{0.1cm} \left(\eta^2\displaystyle\int_{\Omega}v_1(x)^2dx\right)^{1/2}\\ \\
		&=& \sqrt{2}\left( K^{2}+(\mu+\eta)^{2}\right)^{1/2}\Vert v_1\Vert_{\mathbb{L}^{2}(\Omega)}\\
		& & + \hspace{0.1cm} 2 \left( K^{2}\Vert v_1\Vert_{\mathbb{L}^{2}(\Omega)}^{2}+ c_1^2\Vert v_3\Vert_{\mathbb{L}^{2}(\Omega)}^2+c_2^2\Vert v_4\Vert_{\mathbb{L}^{2}(\Omega)}^2+(\rho\mu)^2\Vert v_1\Vert_{\mathbb{L}^{2}(\Omega)}^2\right)^{1/2} \\
		& & + \hspace{0.1cm} \sqrt{2} \left(  c_1^2\Vert v_3\Vert_{\mathbb{L}^{2}(\Omega)}^2+((1-\rho)\mu)^2\Vert v_1\Vert_{\mathbb{L}^{2}(\Omega)}^2 \right)^{1/2}+c_2\Vert v_4\Vert_{\mathbb{L}^{2}(\Omega)}+\eta \Vert v_1\Vert_{\mathbb{L}^{2}(\Omega)}.
	\end{eqnarray*}
	We use the fact that $\sqrt{a+b}\leq \sqrt{a}+\sqrt{b}$ for $a,b\geq 0$, we get that 
	
	\begin{eqnarray*}
		\Vert F(v)\Vert  &\leq & \left( \sqrt{2\left( K^{2}+(\mu+\eta)^{2}\right)} + \eta + 2K+\sqrt{2}((1-\rho)\mu)^2 \right)\Vert v_1\Vert_{\mathbb{L}^{2}(\Omega)}\\
		& & + \hspace{0.1cm} \sqrt{2}\left(\sqrt{2}+1\right)c_1 \Vert v_3\Vert_{\mathbb{L}^{2}(\Omega)}+3 c_2 \Vert v_4\Vert_{\mathbb{L}^{2}(\Omega)}\\
		&\leq & C_1 \Vert v\Vert,
	\end{eqnarray*}
	
	where
	
	\begin{eqnarray*}
		C_1&=&\max\left[1, \sqrt{2\left( K^{2}+(\mu+\eta)^{2}\right)} + \eta + 2K+\sqrt{2}((1-\rho)\mu)^2), \sqrt{2}\left(\sqrt{2}+1\right)c_1, 3 c_2 \right].
	\end{eqnarray*}
\end{proof}

\begin{lemma}\itshape For each $r>0$, there exists a positive constant $ C_2(r)>0$ such that 
	\begin{center}
		$\Vert F(u)-F(v)\Vert \leq C_2(r)\Vert u-v\Vert $, \text{ for } $u,v\in B(0,r)\cap \Lambda_{+}$,
	\end{center}
	where $B(0,r)$ is the open ball (in $X$) of center $0$ and radius $r$.
	\label{Lemma 5.2}
\end{lemma}
\begin{proof} The proof follows the fact that $F$ is $\mathcal{C}^{1}$ from $\Lambda^{+}$ to $X$.
\end{proof}
\begin{theorem} \itshape
	Operator $(A,D(A))$ generates a $C_0$-semigroup $(\mathcal{T}(t))_{t\geq 0}$ of contractions on $X$.
\end{theorem}
\begin{proof}
	We show that $-A$ is maximal monotone. Let $D_1=D_b$, $D_2=D_n$, $D_3=D_4=0$, and $D_5=D_c$. Let $v=(v_1,\cdots,v_5)\in D(A)$, then
	\begin{eqnarray*}
		\langle v,-Av\rangle &=& \sum\limits_{i=1}^{5}\displaystyle\int_{\Omega}-D_i\Delta v_i(x)v_i(x)dx\\
		&=&  \sum\limits_{i=1}^{5}D_i\displaystyle\int_{\Omega}\vert \nabla v_i(x)\vert^{2} dx\\
		&=& \sum\limits_{i=1}^{5}D_i \Vert \nabla v_i\Vert^{2}_{\mathbb{L}^{2}(\Omega)}.
	\end{eqnarray*}
	Hence, $-A$ is monotone. Since, $rg(I-D_i\Delta)=\mathbb{L}^{2}(\Omega)$ for $i=1,2,\cdots,5$, we affirm that for each $v=(v_1, v_2,v_3,v_4,v_5)\in X$ there exists $w=(w_1, w_2,w_3,w_4,w_5)\in D(A)$ such that $(I- A)w=v$. Thus, $rg(I- A)=X$, that is $-A$ is maximal. Consequently, $(A,D(A))$ is an infinitesimal generator of a $C_0$-semigroup $(\mathcal{T}(t))_{t\geq 0}$ of contractions on  $X$.
\end{proof}
\begin{theorem} \cite{bataki}
	$(\mathcal{T}(t))_{t\geq 0}$ is positive, namely $\mathcal{T}(t)\Lambda_{+}\subseteq \Lambda_{+}$ for all $t\geq 0$.
	\label{thm 7}
\end{theorem}
\begin{remark} \itshape For each $\lambda>0$, operator $(A-\lambda I,D(A))$ generates a $C_0$-semigroup $(\mathcal{S}(t))_{t\geq 0}$ given by 
	\begin{center}
		$\mathcal{S}(t)=e^{-\lambda t}\mathcal{T}(t)$, \text{ for } $t\geq 0$.
	\end{center}
\end{remark}

\begin{definition}
	A continuous function $u:\mathbb{R}^{+}\to X$ is said to be a mild solution of equation (\ref{eq 3.2}) if 
	\begin{center}
		$u(t)=\mathcal{S}(t)u_0+\displaystyle\int_{0}^{t}\mathcal{S}(t-s)G(u(s))ds$, \text{ for } $t\geq 0$.
	\end{center}
\end{definition}

We recall the following Theorem.
\begin{theorem} \cite{13} Let $f:X\to X$ be locally Lipschitz continuous and at most affine. If $B$ is the infinitesimal generator of a $C_0$-semigroup $(T(t))_{t\geq 0}$ on $X$, then for every $v_0\in X$,  the initial value problem
	\begin{equation*}
	\left\{\begin{array}{l}
	v'(t)=B v(t)+f(v(t)), \quad t\geq 0\\
	v(0)=v_0,
	\end{array}\right.
	\end{equation*}
	has a unique mild solution $v$ on $[0,+\infty[$.
	\label{existence}
\end{theorem}

The following Theorem is the main result in this section. 
\begin{theorem} \itshape Assume that $u_0\in \Lambda_{+}$. Equation $(\ref{eq 3.2})$ has a unique positive mild solution defined on $[0,+\infty[$.
\end{theorem}

\begin{proof}
	The existence and uniqueness of the solution over the interval $[0, +\infty[$ can be readily established based on the findings of Lemma \ref{lem 3.1}, Lemma \ref{Lemma 5.2}, and Theorem \ref{existence}. Additionally, as noted in Remark \ref{rem 2}, we can choose a sufficiently large positive value for $\lambda$ to ensure the positivity of the function $G$ on $\Lambda_{+}$. Subsequently, proving the positivity of the solution reduces to Theorem \ref{thm 7}.
\end{proof}
\section{Global attractor for equation \eqref{eq 1.1}}\label{Sec 6}
\noindent

Let define the family of mapping $\{\mathcal{U}(t)\}_{t\geq 0}$ on $\Lambda_{+}$ by 
\begin{center}
	$\mathcal{U}(t)u_0:=u(t,u_0):=u(t)$, \text{ for } $t\geq 0$,
\end{center}
where $u(t)$ is the mild solution of equation \eqref{eq 3.2} at time $t$ corresponding to the initial condition $u_0$. It follows immediately that $\mathcal{U}(0)u_0=u_0$ and $t\to \mathcal{U}(t)u_0$ is continuous for all $u_0\in \Lambda_{+}$. Since \eqref{eq 3.2} is autonomous, the property
\begin{center}
	$\mathcal{U}(t+s)=\mathcal{U}(t)\mathcal{U}(s)$, \text{ for } $t\geq s\geq 0$,
\end{center}
is a consequence of the uniqueness of solution. As a consequence, $(\mathcal{U}(t))_{t\geq 0}$ is a semiflow on $\Lambda_{+}$. The following result shows that $\mathcal{U}(t)$ is a bounded map for each $t\geq 0$.
\begin{proposition} There exists $\tilde{C}_1>1$ such that $\Vert \mathcal{U}(t)u_0\Vert \leq \tilde{C}_1\Vert u_0\Vert$ for $t\geq 0$ and all $u_0\in \Lambda_{+}$.
	\label{thm 5.2}
\end{proposition}
\begin{proof}
	Let $u_0\in \Lambda_{+}$ be fixed. Then, for each $t\geq 0$, we have
	\begin{eqnarray*}
		\Vert \mathcal{U}(t)u_0\Vert &\leq & e^{-\lambda t } \Vert u_0\Vert+\displaystyle\int_{0}^{t}e^{-\lambda(t-s)}\Vert G\left( \mathcal{U}(s)u_0 \right)\Vert ds\\
		&\leq & e^{-\lambda t } \Vert u_0\Vert+\displaystyle\int_{0}^{t}e^{-\lambda(t-s)}(\lambda+C_1)\Vert  \mathcal{U}(s)u_0\Vert ds\\
		&\leq & \Vert u_0\Vert+ \displaystyle\int_{0}^{t}w(t,s)\Vert  \mathcal{U}(s)u_0\Vert ds,
	\end{eqnarray*}
	where
	\begin{center}
		$w(t,s)=e^{-\lambda(t-s)}(\lambda+C_1)$, \text{ for } $t\geq s\geq 0$.
	\end{center}
	By Lemma \ref{lem 2.2}, we obtain that
	\begin{eqnarray*}
		\Vert \mathcal{U}(t)u_0\Vert & \leq &  \exp\left(\displaystyle\int_{0}^{t}\left[(\lambda+C_1)+\displaystyle\int_{0}^{s}\dfrac{\partial w(s,r)}{\partial s} dr\right]ds\right)\Vert u_0\Vert\\
		&=& \exp\left(\displaystyle\int_{0}^{t}(\lambda+C_1)\left[1-\lambda\displaystyle\int_{0}^{s}e^{-\lambda(s-r)} dr\right]ds\right)\Vert u_0\Vert\\
		&=&  \exp\left(\displaystyle\int_{0}^{t}(\lambda+C_1)\left[1-(1-e^{-\lambda s})\right]ds\right)\Vert u_0\Vert\\
		&=& \exp\left(\displaystyle\int_{0}^{t}(\lambda+C_1)e^{-\lambda s}ds\right)\Vert u_0\Vert\\
		&=& \exp\left(\dfrac{(\lambda+C_1)}{\lambda}(1-e^{-\lambda t})\right)\Vert u_0\Vert\\
		&\leq & \tilde{C}_1\Vert u_0\Vert,
	\end{eqnarray*}
	where 
	\begin{center}
		$\tilde{C}_1=\exp\left( 1+\dfrac{C_1}{\lambda} \right)$.
	\end{center}
\end{proof}
\begin{theorem} \itshape  The semiflow $(\mathcal{U}(t))_{t\geq 0}$ is point dissipative, i.e,
	there exists $\tilde{C}_2>0$  such that $\limsup\limits_{t\to +\infty}\Vert \mathcal{U}(t)u_0\Vert\leq \tilde{C}_2$ for all $u_0\in \Lambda_{+}$.
	\label{thm 5.1}
\end{theorem}
\begin{proof} Let $u_0\in \Lambda_{+}$ be fixed, and $t_0:=t_0(u_0)>0$ such that $\Theta(t):=e^{-\lambda t }  \left[ 1+ \dfrac{(\lambda+C_1)\tilde{C}_1}{\lambda} \left(e^{\lambda t_0}-1\right)\right]\Vert u_0\Vert\leq 1$ for $t> t_0$. Then, for each $t> t_0$, we have 
	\begin{eqnarray*}
		\Vert \mathcal{U}(t)u_0\Vert &\leq& e^{-\lambda t } \Vert u_0\Vert+ \displaystyle\int_{0}^{t_0}e^{-\lambda(t-s)}(\lambda+C_1)\Vert  \mathcal{U}(s)u_0\Vert ds\\
		& & +\displaystyle\int_{t_0}^{t}e^{-\lambda(t-s)}(\lambda+C_1)\Vert  \mathcal{U}(s)u_0\Vert ds.
	\end{eqnarray*}
	By Proposition \ref{thm 5.2}, we obtain that
	\begin{eqnarray*}
		\Vert \mathcal{U}(t)u_0\Vert	&\leq & e^{-\lambda t } \Vert u_0\Vert+(\lambda+C_1)\tilde{C}_1\left(\dfrac{e^{-\lambda(t-t_0)}-e^{-\lambda t}}{\lambda}\right)\Vert u_0\Vert\\
		& & +\displaystyle\int_{t_0}^{t}e^{-\lambda(t-s)}(\lambda+C_1)\Vert  \mathcal{U}(s)u_0\Vert ds\\
		& = & \Theta(t)+\displaystyle\int_{t_0}^{t}e^{-\lambda(t-s)}(\lambda+C_1)\Vert  \mathcal{U}(s)u_0\Vert ds\\
		&\leq & 1+ \displaystyle\int_{t_0}^{t}w(t,s)\Vert  \mathcal{U}(s)u_0\Vert ds,
	\end{eqnarray*} 
	where
	\begin{center}
		$w(t,s)=e^{-\lambda(t-s)}(\lambda+C_1)$, \text{ for } $t\geq s\geq t_0$.
	\end{center}
	By Lemma \ref{lem 2.2}, we obtain that
	\begin{eqnarray*}
		\Vert \mathcal{U}(t)u_0\Vert &\leq& \exp\left(\displaystyle\int_{t_0}^{t}\left[(\lambda+C_1)\left[1-\lambda\displaystyle\int_{t_0}^{s}e^{-\lambda(s-r)} dr \right] \right]ds\right)\\
		&=& \exp\left(\displaystyle\int_{t_0}^{t}\left[(\lambda+C_1)\left[1- e^{-\lambda s}(e^{\lambda s}-e^{\lambda t_0}) \right] \right]ds\right)\\
		&=& \exp\left(\displaystyle\int_{t_0}^{t}\left[(\lambda+C_1)e^{-\lambda (s-t_0)} \right]ds\right)\\
		&=& \exp\left(\dfrac{(\lambda+C_1)}{\lambda}\left(1-e^{-\lambda (t-t_0)}\right)\right)\\
		&\leq& \exp\left(\dfrac{\lambda+C_1}{\lambda}\right).
	\end{eqnarray*}
	Consequently, for each $u_0\in \Lambda_{+}$, we have 
	\begin{center}
		$\limsup\limits_{t\to +\infty}\Vert \mathcal{U}(t)u_0\Vert \leq \tilde{C}_2$, where $\tilde{C}_2=\exp\left(\dfrac{\lambda+C_1}{\lambda}\right)$.
	\end{center}
\end{proof}

The following is the main result in this Section, and it is an immediate consequence of Proposition \ref{thm 5.2}, Theorem \ref{Theorem 4.1.2},  and Theorem \ref{thm 5.1}.
\begin{theorem} \itshape There exists a global attractor $\mathcal{A}$ corresponding to the semiflow $(\mathcal{U}(t))_{t\geq 0}$. More precisely, $\mathcal{A}$ is a closed and bounded invariant subset of $\Lambda_{+}$ that attracts all bounded subsets of $\Lambda_{+}$.
\end{theorem}
\section{Numerical vizualization of the global attractor} \label{Sec 7}
\noindent

Visualizing the global attractor of a partial differential equation through numerical simulations can be complex. The complexity comes from several factors. First, PDEs often describe systems with high-dimensional state spaces, requiring a large number of variables to accurately represent the system's behavior. As a result, numerical simulations need to discretize the state space, leading to a substantial increase in computational resources required as the dimensionality grows. Second, the evolution of the PDEs over time demands solving a system of differential equations numerically on complex geometries, typically using methods such as finite difference, finite element, or graph based methods. These numerical schemes involve computations that scale with the size of the discretization, leading to increased computational costs for larger systems or complex geometries.  Additionally, the time span required to observe the long-term behavior of the attractor might be extensive, requiring prolonged simulations.  Consequently, exploring the global attractor computationally often necessitates substantial computational resources, efficient numerical algorithms, and parallel computing techniques to tackle the computational complexity and achieve meaningful visualizations.

In order to visualize the global attractor we simulate the model described above on real sandy loam soil sample captured using micro tomographic imaging, for a comprehensive understanding of the soil samples and the techniques employed to generate CT images, readers are referred to the study conducted by Juyal et al. (2018). 

The pore space is initially represented as a three-dimensional binary image with dimensions of $512\times512 \times 512$. Figure \ref{Fig 1} present a random z plan of the segmented 3D image; the pore space voxels in the image are identified and labeled as black. The 3D image is presented using a uniform resolution of 24 $\mu m$. 

In order to address the computational limitations, we extract a 3D portion of the original image measuring $50^3 pixel^3$ pixels located in the $[[50,100],[50,100],[150,200]]$ region. The porosity of the extracted portion is 0.13\%. To simplify the geometric complexity of the problem, we employ a method outlined in \cite{18} to approximate the pore space using spheres. Subsequently, we utilize the numerical model described in \cite{17} to simulate the long-term behavior of the system within the intrinsic geometry of the resulting pore network model. The numerical method achieves efficient computation time and generates highly accurate simulations of the model described earlier in a complex network of geometric primitives. In this discussion, we first provide an overview of the geometric technique employed to extract the minimal set of maximal spheres that cover the pore space. Then, we delve into the specifics of the numerical framework utilized for simulating the global attractor of the problem.

\begin{figure}[h]
	\centering
	\includegraphics[width=8cm]{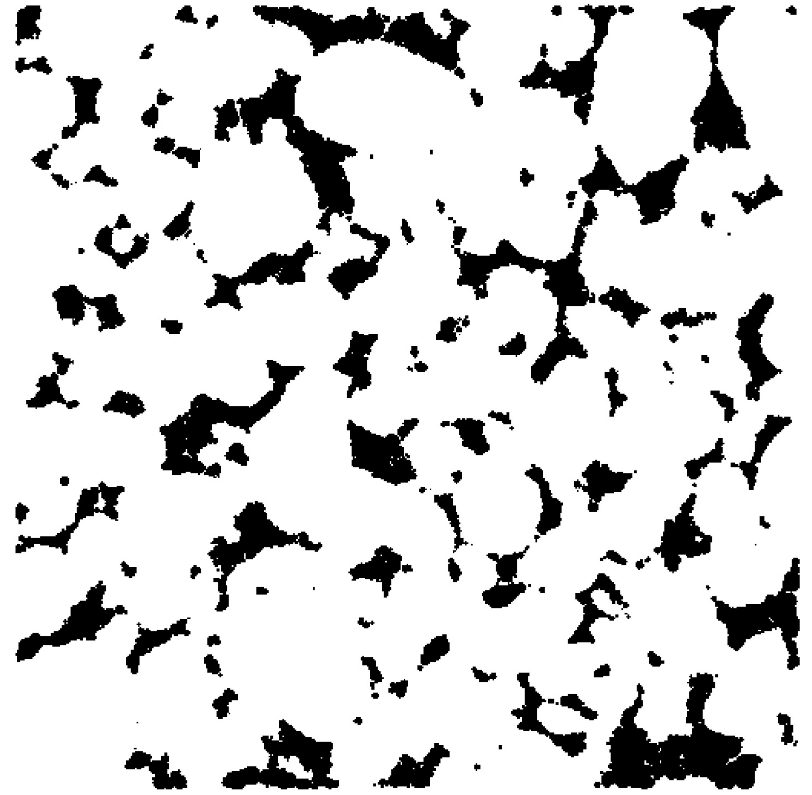}
	\caption{A random z plan of the 3D binary image of the sandy loam soil sample, the pore space in the image is identified with the black color.}
	\label{Fig 1}
\end{figure}

\subsection{Pore network extraction}
\noindent

To facilitate reader comprehension, we present a summarized description of the geometric modeling of the pore space as a network of maximal spheres that cover the pores of the 3D image.

The methodology employed in this study, as detailed in \cite{18}, involves utilizing a minimal set of balls to reconstruct the skeleton of the shape. This approach offers a more concise representation of the pore space, which is better suited for numerical simulations compared to the original voxel set \cite{17}. This representation is also considered more realistic compared to idealized pore network models, as highlighted in the work by \cite{19}. In our research, we adopt balls as the primary primitives for this purpose. However, it is worth noting that alternative primitives, such as ellipsoids, could have been incorporated within the same numerical simulation framework \cite{20}.

Using the aforementioned primitives, we construct an attributed adjacency valuated graph that accurately represents the pore space. In this graph, each node corresponds to a sphere, representing the concept of a pore, while each arc signifies an adjacency between two spheres.

By considering the collective set of primitives, we obtain an approximation of the overall pore space, allowing for a comprehensive analysis of connectivity and relationships between different pore-like structures.

Figure \ref{Fig 2} illustrates the pore network corresponding to the selected 3D section using a Matlab routine.
\begin{center}
	\begin{figure}[h] 
		\includegraphics[width=13cm]{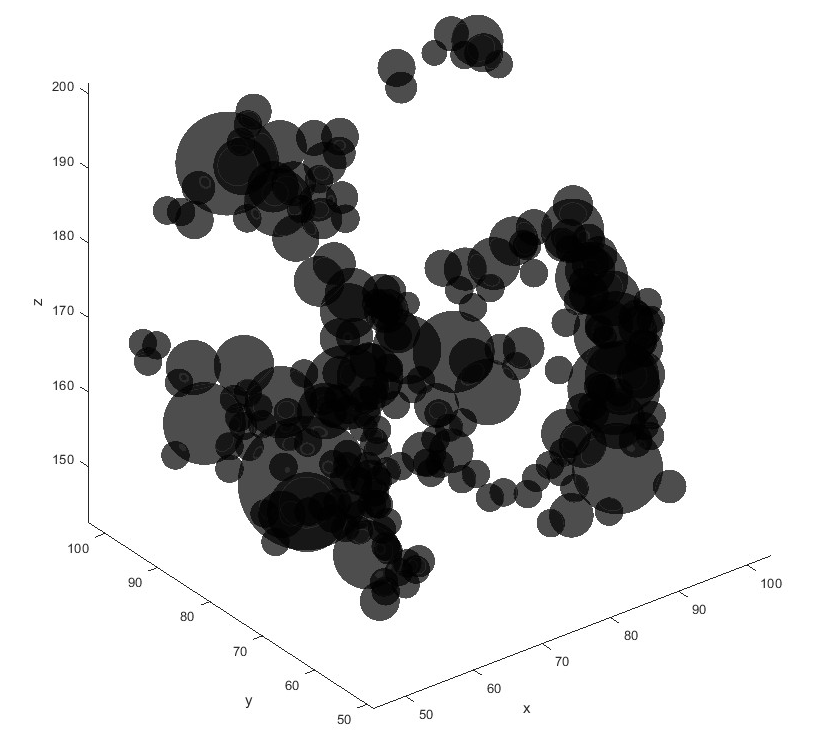}
		\caption{The minimal set of maximal sphere network covering the pore space of the selected portion from the original 3D binary image}
		\label{Fig 2}
	\end{figure}
\end{center}

\subsection{Numerical schemes of the global model in the extracted pore network}
\noindent

As in \cite{17}, we divide the dynamics to diffusion and transformation processes which suit very well for the graph-based representation of the domain.
The pore space is presented using a graph of geometric primitives; balls in this experience (see Figure \ref{Fig 2}). 

So given a set $ \textbf{P} = \{ p_1,\cdots,p_n \} $ of geometric primitives obtained using the approach described above [Olivier Monga, 2007], we construct a graph G(N,E), where:
$N = \{1,\cdots,n \} $ is the index set of the graph which corresponds to the primitives in P, and
$E = \{ (i,j) / p_i \cap p_j \neq \emptyset \}$  is the set of edge indices which encodes the geometrical adjacency between the primitives.

Let $t\geqslant 0$, for each node i of N we attach the following vector $$X_i(t):= \{x_i, y_i, z_i, v_i, b^{i}(t),n^{i}(t),m_1^{i}(t),m_2^{i}(t),c^{i}(t) \},$$ where $(x_i, y_i, z_i)$ are the coordinates of the gravitation center of the geometrical primitive $p_i$ and $v_i$ its volume, $b^{i}(t),n^{i}(t),m_1^{i}(t),m_2^{i}(t),c^{i}(t)$ are respectively the total mass of MB, DOM, SOM, FOM, CO$_2$ contained within the primitive $p_i$ at time t.  

We firstly present the mathematical model for diffusion in the set of geometrical primitives, then we adjust the global model in order to give the general framework for simulating the global mathematical model in the graph of geometrical primitives. 

\subsubsection{Diffusion processes}
\noindent

Fick's first Law of Diffusion describes the rate of diffusion through a medium in terms of the concentration gradient. It states that the flux (J) of particles across a unit area perpendicular to the direction of diffusion is proportional to the concentration gradient .

Specifically, the concentration gradient at a given point  in the contact area between two adjacent regular (geometric regularity) primitives $p_i$ and $p_j$ can be expressed as follows, 
$$ \frac{\frac{m_i}{v_i}-\frac{m_j}{v_j}}{L_{i,j}}, $$
where $\frac{m_i}{v_i}$ is the concentration in the primitive $p_i$ and $\frac{m_j}{v_j}$ is the concentration at the primitive $p_j$, while $L_{i,j}$ is the distance between the centers of the primitives, and since the primitives are regular we take the distance between gravitation centers. 

Then, the total flux of mass from the primitive  $p_j$ to the primitive  $p_i$  during a time period $\delta t$ is the following,
$$ J_{i,j} = - D A_{i,j} \delta t \frac{c_i-c_j}{L_{i,j}}, $$
where $A_{i,j}$ is the area of contact between the two primitives, D is the diffusion coefficient of the compound to be diffused and $c_i = \frac{m_i}{v_i}$.

The total variation of mass by diffusion processes at the primitive $p_i$ during a time period $\delta t$ is the sum of all the fluxes coming from the neighboring primitives, which is expressed as follows;

$$\\
\delta m_i = \sum_{j \in N / (i,j) \in E} - D A_{i,j} \delta t \frac{c_i-c_j}{L_{i,j}},
$$
which results in , 
\begin{align*}
\frac{dm_i}{dt}(t) & = \sum_{j \in N / (i,j) \in E} - D A_{i,j} \frac{c_i(t)-c_j(t)}{L_{i,j}} \\
& =D ( c_i(t)\sum_{j \in N / (i,j) \in E} - Q_{i,j}  + \sum_{j \in N / (i,j) \in E}Q_{i,j} c_j(t)) \\
& =D (\frac{m_i(t)}{v_i}\sum_{j \in N / (i,j) \in E} - Q_{i,j}  + \sum_{j \in N / (i,j) \in E}Q_{i,j} \frac{m_j(t)}{v_j} ), 
\end{align*}
where $ Q_{i,j} = \frac{A_{i,j}}{L_{i,j}}$.

Gathering the masses $m_i$ of all the primitives in the same vector $ M(t) = (m_i(t))_{i\in N}$, we get the following equation, 
$$
\frac{dM}{dt}(t)  = D \hat{\bigtriangleup} M(t),
$$

where $ \hat{\bigtriangleup}_{i,i} = \frac{ 1}{v_i} \sum_{j \in N / (i,j) \in E} - Q_{i,j}  $, and $ \hat{\bigtriangleup}_{i,j} = \frac{ Q_{i,j} }{v_j} $ for $ i \neq j$.

\subsection{Transformation processes and general model}
\noindent

Similarly to the model described above we consider that FOM and SOM are decomposed rapidly and slowly, respectively. DOM comes from the hydrolysis of SOM and FOM. DOM diffuses through water paths (water-filled spheres) and is consumed by MB for its growth. We hypothesized that MB does not move. Dead microorganisms are recycled into SOM and DOM. MB respires by producing inorganic carbon (CO$_2$). 

The changes of the set of biological features $(b^{i}(t),n^{i}(t),m_1^{i}(t),m_2^{i}(t),c^{i}(t))_{i \in N} $, due to transformation of different compounds and diffusion processes of DOM in the water-filled primitives, within a time step $\delta t$  are expressed using the following system of equations \eqref{Eq 6}:

\begin{equation}
\quad \forall i \in N, \hspace{0.3cm}\left\{
\begin{aligned}
\quad \frac{db^{i}}{dt}(t)  &=   \frac{K n^{i}(t)}{ K_b + n^{i}(t)} b^{i}(t)- (\eta+\mu) b^{i}(t)\\
\quad \frac{dn^{i}}{dt}(t)  &= D_n (\hat{\bigtriangleup} N(t))_i   +  \rho \mu b^{i}(t)  -  \frac{K n^{i}(t)}{ K_b + n^{i}(t)} b^{i}(t)\\ & \quad + c_1 m_1^{i}(t) + c_2 m_2^{i}(t)  \\
\quad \frac{dm_1^{i}}{dt}(t)  &=  - c_1 m_1^{i}(t) +(1 - \rho) \mu b^{i}(t) \\ 
\quad\frac{dm_1^{i}}{dt}(t)  &=  -c_2 m_2^{i}(t)   \\ 
\quad \frac{dc^{i}}{dt}(t)  &= \eta b^{i}(t).
\end{aligned}\right.
\label{Eq 6}
\end{equation}

In the given context, $N(t) =(n^i(t))_{i\in N}$ denotes the vector of the contained DOM mass within all the geometric primitives at time t, $\eta$ represents the relative respiration rate in units of $day^{-1}$, $\mu$ denotes the relative mortality rate in units of $day^{-1}$, $\rho$ signifies the proportion of MB that returns to DOM while the remaining fraction returns to SOM. Furthermore, $c_2$ and $c_1$ correspond to the relative decomposition rates of FOM and SOM respectively, both measured in units of $day^{-1}$. Additionally, $K$ and $K_b$ represent the maximum relative growth rate of MB and the constant of half-saturation of DOM by MB, respectively, both measured in units of $day^{-1}$ and gC (grams of carbon).

The term $D_n(\hat{\bigtriangleup}N(t))_i$ represents the mass variation of DOM caused by the exchange of mass between the primitive $p_i$ and all the connected primitives. Here, $D_n$ denotes the molecular diffusion coefficient of DOM in water, measured in units of $cm^2 . d^{-1}$.
\subsection{Aggregation Techniques for Visualizing the Attractors}
\noindent

In the context of visualizing attractors for high dimensional partial differential equations, it is often necessary to use aggregation techniques to simplify and clarify the data. Aggregation involves summarizing multiple variables or data points into a single value, which can help to reduce the dimensionality of the system and create more effective visualizations. 

To better understand our model, we wanted to look at how the biological parameters change over time. We focused on studying the long-term growth of microorganisms in relation to their need for organic matter and the production of mineralized carbon (CO$_2$).

In order to illustrate we introduce the following application $AGG_1$ that for a distribution of $L^2 (\Omega)$ returns the integral sum over omega of the distribution, mathematically defined as,

\begin{align*}
AGG_1 \colon \Lambda_{+} &\to \mathbb{R}_+^5\\
(\phi_1,\cdots,\phi_5) & \mapsto (\int_{\Omega} \phi_1,\cdots,\int_{\Omega} \phi_5) .
\end{align*}

This function calculates the overall mass of biological parameters in a given sample by integrating their distribution across the domain $\Omega$.

In the numerical method, each primitive $p_i$ is expressed by the contained masses 
$\{ b^{i}(t),n^{i}(t),m_1^{i}(t),m_2^{i}(t),c^{i}(t)\}$  within it's volume $v_i$ at time t. 

So, the total mass of different biological parameters within the pore space is obtained by the following;
\begin{align*}
AGG_1 \colon &  \quad \quad \quad \quad (l^2(N))^5  &    &\to&  \mathbb{R}_+^5 &\\
& [(b^{i})_{i \in N},(n^{i})_{i \in N},(m_1^{i})_{i \in N},(m_2^{i})_{i \in N},(c^{i})_{i \in N}] &  &\mapsto& [\bar{B},\bar{N},\bar{M_1},\bar{M_2},\bar{C}],
\end{align*}
where $$[\bar{B},\bar{N},\bar{M_1},\bar{M_2},\bar{C}] := \left[\sum_{i \in N} b^{i} ,\sum_{i \in N} n^{i},\sum_{i \in N} m_1^{i} ,\sum_{i \in N} m_2^{i},\sum_{i \in N} c^{i}\right].$$

Despite the fact that this approach reduces the attractor from the infinite-dimensional space $(l^2(N))^5$ to a subset in $\mathbb{R}^5$, we further aggregate the data to focus solely on the microorganisms growth affected by the total organic matter and the carbon dioxide occurring in the soil sample. We also, use projection in order to focus on specific parameters. For this purpose we introduce the following applications:

\begin{align*}
AGG_2 \colon &  \quad \quad \quad \quad \mathbb{R}_+^5& &\to& \mathbb{R}_+^3&  \\
& [\bar{B},\bar{N},\bar{M_1},\bar{M_2},\bar{C}]& &\mapsto& [\bar{B},\bar{N}+\bar{M_1}+\bar{M_2},\bar{C}]
\end{align*}

and 

\begin{align*}
\mathcal{P} \colon &  \quad \quad \quad \quad \mathbb{R}_+^5&  &\to& \mathbb{R}_+^3&  \\
& [\bar{B},\bar{N},\bar{M_1},\bar{M_2},\bar{C}]&  &\mapsto& [\bar{B},\bar{N},\bar{C}].
\end{align*}

In the given context, the variable $\bar{B}$ represents the overall population of microorganisms present in the soil sample. Similarly, the quantity $\bar{N}+\bar{M_1}+\bar{M_2}$ corresponds to the combined mass of organic matter encompassing its various forms, including dissolved organic matter (DOM), fresh organic matter (FOM), and soil organic matter (SOM). Lastly, $\bar{C}$ represents the amount of carbon that has been released through microbial respiration.
\subsection{Model parameters and initial conditions modeling}
\noindent

The same biological parameters  of Arthrobacter sp. 9R  as in \cite{17,21,22} were utilized. The parameter values assigned were as follows: $\eta$, representing the relative respiration rate, was set to 0.2 $day^{-1}$; $\mu$, representing the relative mortality rate, was set to 0.5 $day^{-1}$; $\rho$, indicating the proportion of Microbial Mass (MB) returning to Dissolved Organic Matter (DOM), was set to 0.55 (with the remaining portion returning to Soil Organic Matter [SOM]); $c_2$ and $c_1$, representing the relative decomposition rates of FOM and SOM, were respectively set to 0.3 $day^{-1}$ and 0.01 $day^{-1}$; $K$, denoting the maximum relative growth rate of MB, was set to 9.6 $day^{-1}$; $K_b$, signifying the constant of half saturation of DOM, was set to 0.001 $gC.g^{-1}$ . 

Different scenarios were conducted on the extracted portion of the original 3D image to explore various initial conditions. Dissolved organic matter (DOM) and microorganisms were distributed in either a heterogeneous or homogeneous manner, randomly. The resulting masses from each scenario are summarized in Figure \ref{Fig 3}.

\begin{center}
	\begin{figure}
		\includegraphics[width=13cm]{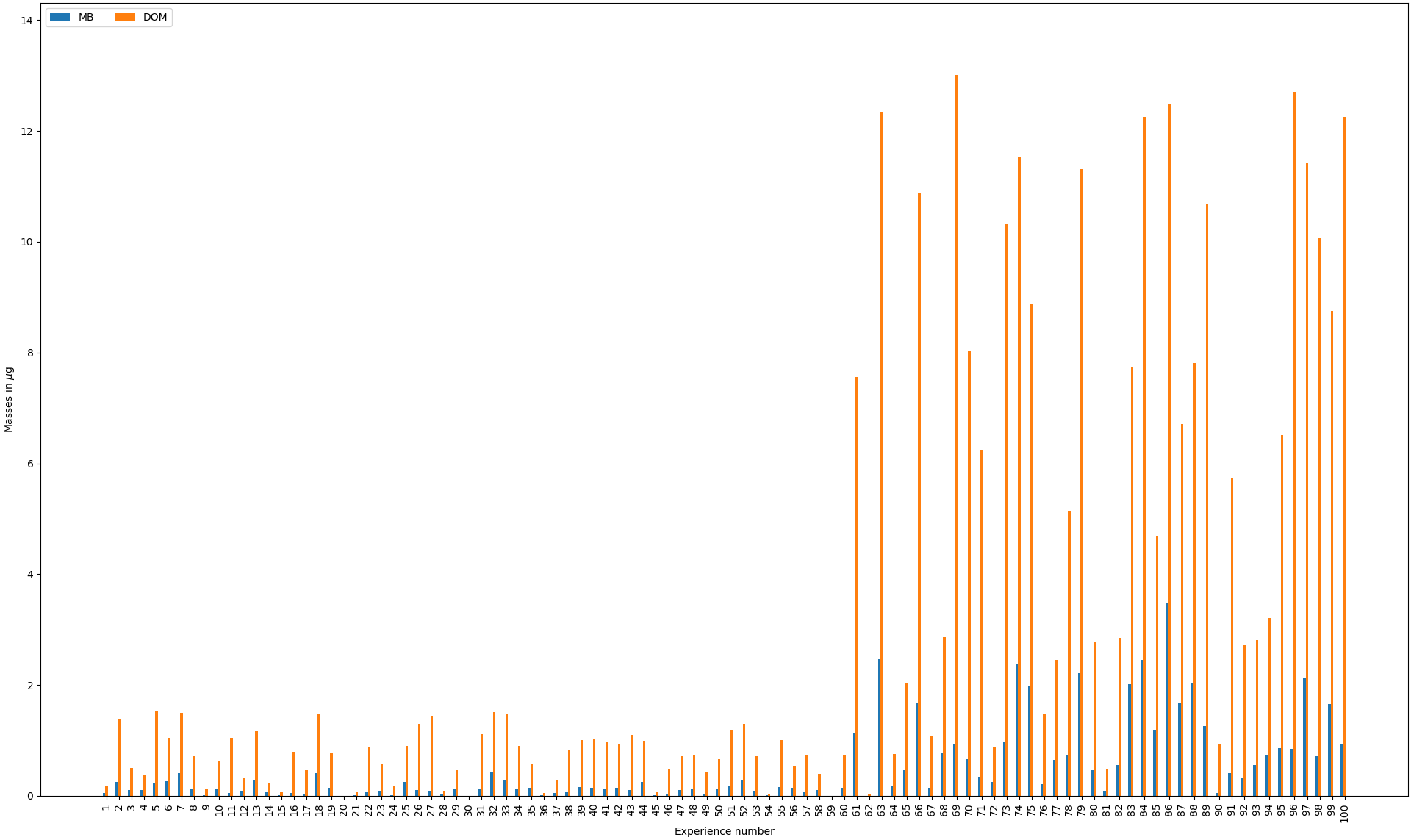}
		\caption{In each scenario, the initial masses of microorganisms (MB) and dissolved organic matter (DOM) are represented, with MB shown in blue and DOM in orange.}
		\label{Fig 3}
	\end{figure}
\end{center}

In each scenario, the mass of DOM in the network of spheres was determined by selecting a value within the range of $[10^{-7} , 9.10^{-4}]$ $\mu g .voxel^{-3}$. The DOM mass was distributed among the water-filled spheres, that represent a volume of $1.73 mm^3$, either heterogeneously or homogeneously.

To distribute the microorganisms in a heterogeneous manner,  following the realistic bacterial model distribution outlined in  \cite{23} and used in previous works  \cite{17,22}, a random value between 0.05\% and 0.15\% of the distributed DOM was chosen to determine the mass of microorganisms. These microorganisms were then distributed in the spheres network as patches, with the number and placement of the patches chosen randomly.

We run simulations of the model from the resulting initial conditions during a period of $918$ days using an implicit scheme of the model \eqref{Eq 6}. Figure \ref{Fig 4}
shows simulations during one year of 4 randomly chosen scenarios.
\begin{center}
	\begin{figure}[h]
		\centering
		\includegraphics[width=13cm]{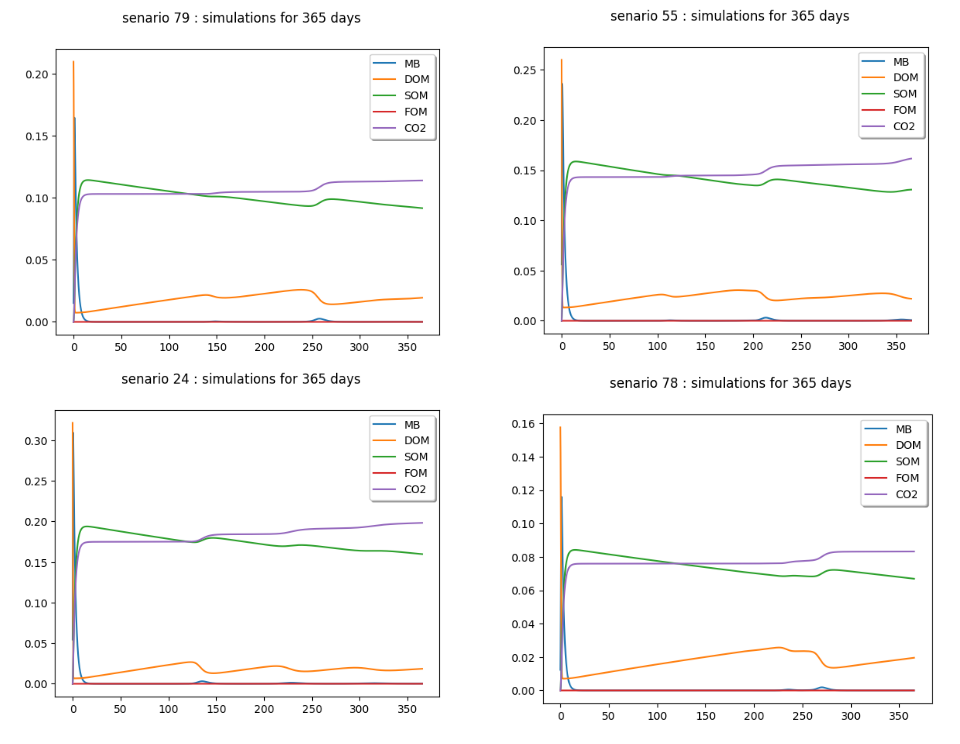}
		\caption{Simulation of some scenarios during one year}
		\label{Fig 4}
	\end{figure}
\end{center}

\section{Results and discussion}  \label{Sec 8}
\noindent
All distributions tend to converge towards a common region and exhibit oscillations within it. The attractor lies in the plane corresponding to a dead population of microorganisms, indicating that microbial growth is limited by carbon availability. This observation is consistent with the experimental findings in \cite{24,25}, which demonstrated carbon limitation in microbial growth.

Figure \ref{Fig 5} shows the long-term patterns of microorganism proliferation and CO$_2$ production, alongside the evolution of total organic matter, including dissolved organic matter (DOM), soil organic matter (SOM), and fresh organic matter (FOM).

Figure \ref{Fig 6} also depicts the long-term evolution of microorganism mass and CO$_2$ production, using the aggregation function $\mathcal{P}$, which considers only the dissolved organic matter (DOM).

Examining Figures \ref{Fig 5} and \ref{Fig 6}, it is clear that the attractor resides within a well-defined region, distinct from the surrounding data points. These visualizations confirm that the attractor is evident in 3D space.

Figure \ref{Fig 7} illustrates the stability of a distribution belonging to the attractor by showing the final state of a random scenario. Our analysis demonstrates that when simulations start from this distribution, they consistently remain within the attractor. This highlights the attractor's role in governing the system's long-term behavior, ensuring that trajectories originating from within it stay tightly clustered around its basin of attraction.

Despite variations among different scenarios, the overall trend shows convergence towards a shared region, with periodic fluctuations within this range. This region aligns with the $\bar{B} = 0$ plane, indicating that microbial growth is restricted due to a lack of accessible organic matter resources.

\begin{landscape}
	\begin{figure}[t]
		\centering
		\includegraphics[width=23cm]{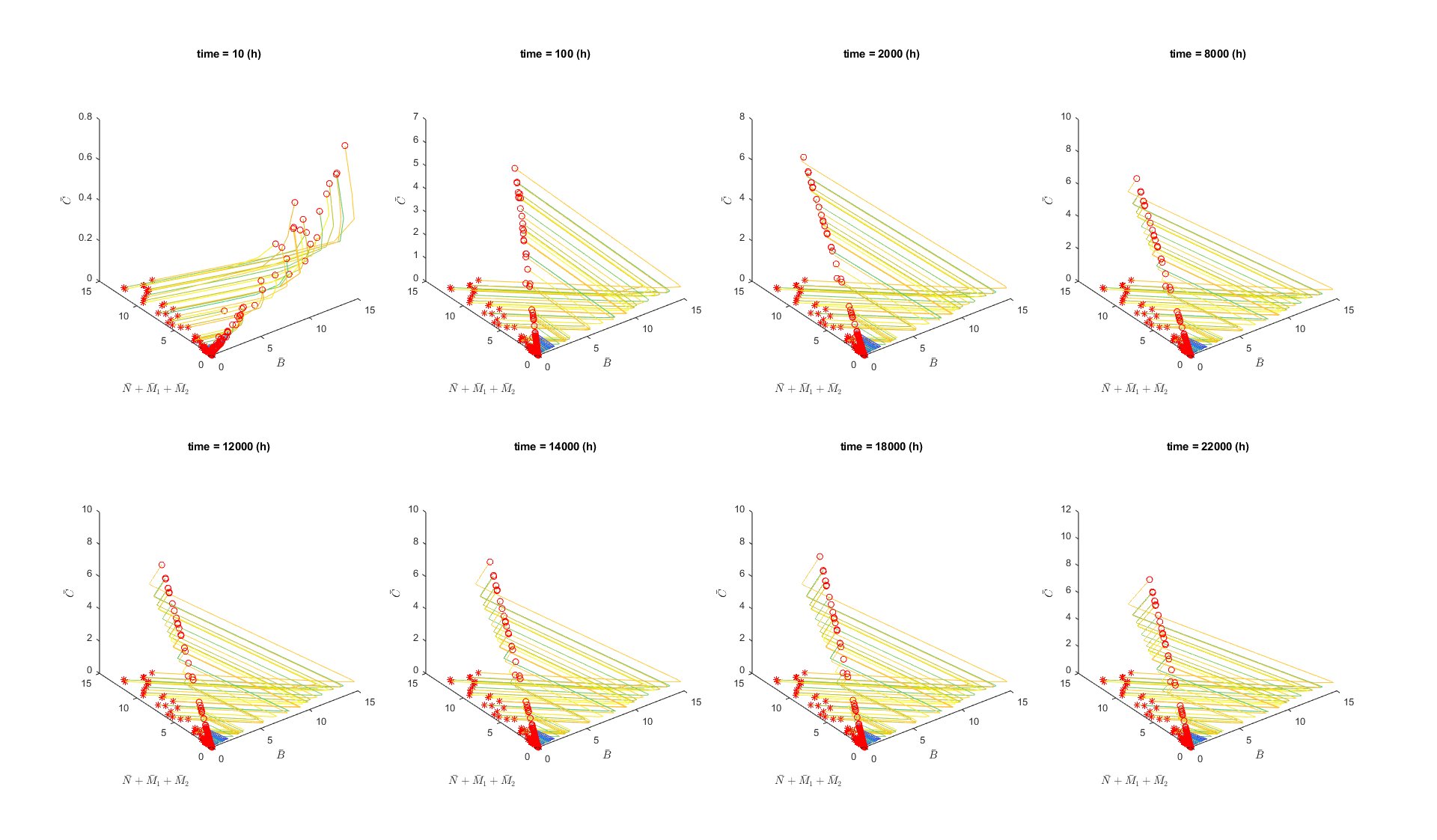}
		\caption{Long-term system dynamics following the application of aggregation functions $AGG_1$ and $AGG_2$. The initial positions are marked with red stars, while the red circles represent the terminal positions in each frame. You can view an animation depicting all simulation frames at this \href{https://youtu.be/VnNqfxjvkME?si=gNXYg-7zheIoEnyu}{link}.}
		\label{Fig 5}
	\end{figure}
	
	\begin{figure}[t]
		\centering
		\includegraphics[width=23cm]{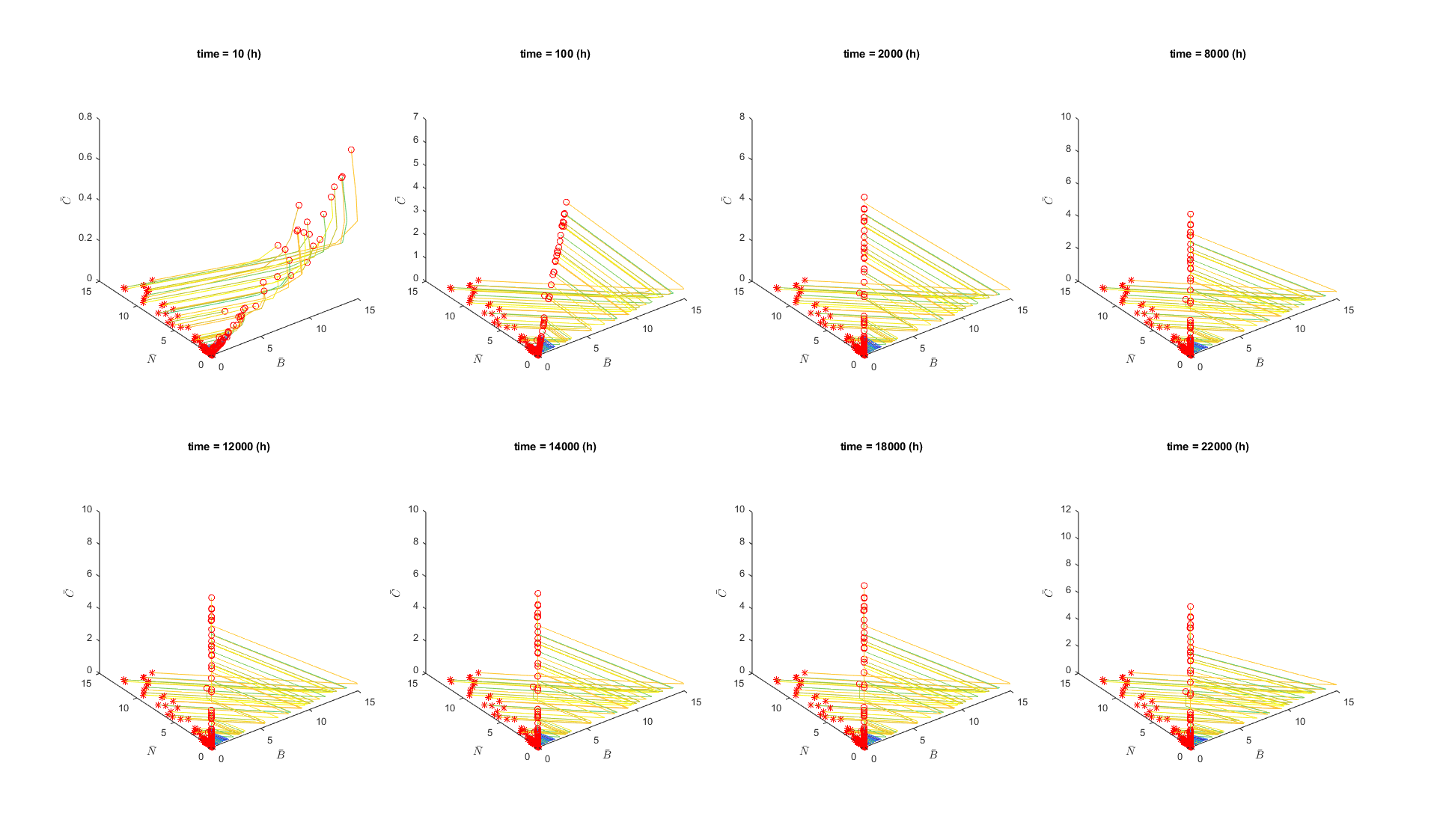}
		\caption{Long-term system dynamics following the application of aggregation functions $AGG_1$ and $\mathcal{P}$. The initial positions are marked with red stars, while the red circles represent the terminal positions in each frame. You can view an animation depicting all simulation frames at this \href{https://youtu.be/Jnf9xkc2JOE?si=kERtZHMy38StF_VA}{link}.}
		\label{Fig 6}
	\end{figure}
	\begin{figure}[t]
		\centering
		\includegraphics[width=23cm]{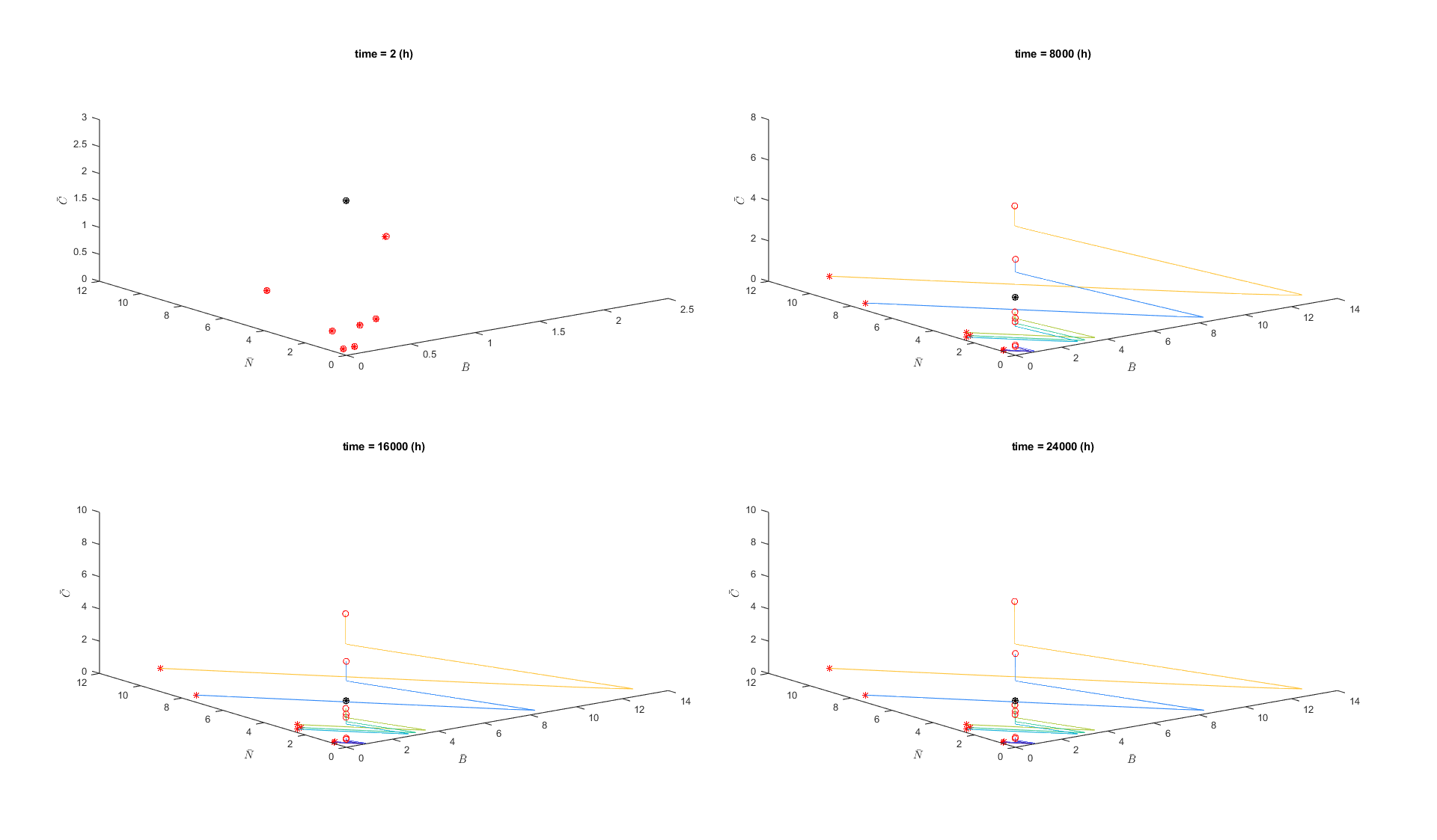}
		\caption{Long-term system dynamics following the application of aggregation functions $AGG_1$ and $AGG_2$. The initial positions are marked with red stars, while the red circles represent the terminal positions in each frame. Notably, the initial and final states of a scenario residing within the attractor are uniquely highlighted by a distinct black star and circle, respectively. }
		\label{Fig 7}
	\end{figure}
\end{landscape}

\section{Conclusion}
In conclusion, this paper analyzed a reaction-diffusion model to describe microbial decomposition of organic matter in a three-dimensional soil structure. We formulated the model using nonlinear parabolic partial differential equations (PDEs) and proved the existence and uniqueness of solutions. Our analysis demonstrated the presence of a global attractor, indicating that solutions converge regardless of initial conditions.

To simulate the model and visualize its long-term behavior, we provided a numerical tool based on modeling the pore space using an optimal sphere network and formulated a numerical scheme that corresponds to the model in the sphere network. This tool is suitable for simulating the model over long periods as it requires reasonable computing power and time.

We provided aggregation functions that reduce the dimensionality of the solution to better visualize the global attractor of the microbial system in soil. The simulations presented in this paper were conducted on a fully saturated real sandy loam soil sample captured using a micro-CT scanner and segmented in the form of a binary 3D image.

We showed that the solution of the model, regardless of the initial distribution, converges to the attractor, which is a region where microorganisms die due to the absence of organic matter.

\section*{CRediT authors contribution statement}
\noindent

Mohammed ELGHANDOURI : Mathematical modeling, mathematical proofs, and manuscript writing.
Mouad KLAI : Numerical formulation and implementation, scenarios modeling, simulations, and manuscript writing.
Khalil EZZINBI : Mathematical methodology and supervision.
Olivier MONGA : Numerical methodology and supervision.


\section*{Acknowledgments}
\noindent

The research described in this article was made possible thanks to the French institute for development through the PhD grants offered to Elghandouri and Klai and thanks to the Moroccan CNRST through the project I-Maroc. We are grateful to Frédéric Hecht for engaging in insightful discussions regarding the visualization of the attractor.

\end{document}